\documentclass[runningheads]{llncs}

\usepackage{tikz}
\usepackage{caption}
\usepackage{bbm}

\usepackage{amssymb}

\usepackage{algorithm,graphicx}
\usepackage[noend]{algpseudocode}

\usepackage{enumitem}
\setitemize{noitemsep,topsep=0.5pt,parsep=0pt,partopsep=0pt}
\setenumerate{noitemsep,topsep=0.5pt,parsep=0pt,partopsep=0pt}

\usepackage{tikzit}

\usetikzlibrary{
  cd,
  math,
  decorations.markings,
  decorations.pathreplacing,
  positioning,
  arrows.meta,
  circuits.logic.US,
  shapes,
  calc,
  fit,
  trees,
  quotes}
\usetikzlibrary{decorations.pathmorphing}
\usetikzlibrary{decorations.markings}
\usetikzlibrary{decorations.pathreplacing}
\usetikzlibrary{arrows}
\usetikzlibrary{shapes.geometric}

\newcommand{\px}{\mathcal{P}_X}
\newcommand{\pxi}{\mathcal{P}_{X_i}}

\newcommand{\sxi}{s_{{i}_{X}}}
\newcommand{\sxj}{s_{{j}_{X}}}
\newcommand{\sxk}{s_{{k}_{X}}}
\newcommand{\sxl}{s_{{l}_{X}}}

\newcommand{\syj}{s_{{j}_{Y}}}
\newcommand{\syjp}{s_{{j'}_{Y}}}

\newcommand{\rl}{\mathbb{R}}

\newcommand{\set}{\mathbf{Set}}

\newcommand{\met}{\mathcal{\mathbf{Met}}}

\newcommand{\metbij}{\met_{bij}}

\newcommand{\prt}{\mathbf{Part}}

\newcommand{\prtbij}{\prt_{bij}}

\newcommand{\bip}{\mathbf{BIP}}

\newcommand{\fprt}[1]{\prt^{\overline{#1}}}
\newcommand{\fprtbij}[1]{\prtbij^{\overline{#1}}}

\newcommand{\fprto}{\fprt{(0,1]^{op}}}

\newcommand{\slink}{\mathcal{SL}}
\newcommand{\slinkr}{\slink_{\mathcal{R}}}

\newcommand{\flatten}{Flatten_{\mu}}

\usepackage{amsmath,amsfonts,bm}

\def\eqref#1{equation~\ref{#1}}

\def\1{\bm{1}}

\DeclareMathAlphabet{\mathsfit}{\encodingdefault}{\sfdefault}{m}{sl}
\SetMathAlphabet{\mathsfit}{bold}{\encodingdefault}{\sfdefault}{bx}{n}

\usepackage{hyperref}
\usepackage{url}

\begin{document}
\title{Flattening Multiparameter Hierarchical Clustering Functors}
%
%


\author{Dan Shiebler 
}
\vspace{-10mm}
\authorrunning{D. Shiebler}
%
\vspace{-10mm}
\institute{University of Oxford}
\maketitle              
\vspace{-5mm}
\begin{abstract}
We bring together topological data analysis, applied category theory, and machine learning to study multiparameter hierarchical clustering. We begin by introducing a procedure for flattening multiparameter hierarchical clusterings. We demonstrate that this procedure is a functor from a category of multiparameter hierarchical partitions to a category of binary integer programs. We also include empirical results demonstrating its effectiveness. Next, we introduce a Bayesian update algorithm for learning clustering parameters from data. We demonstrate that the composition of this algorithm with our flattening procedure satisfies a consistency property.
\end{abstract}

\vspace{-5mm}
\section{Introduction}\label{introduction}\vspace{-3mm}
One of the most useful ways to process a dataset represented as a finite metric space is to cluster the dataset, or break it into groups. An important step is choosing the desired granularity of the clustering, and \textbf{multiparameter hierarchical clustering algorithms} accept a hyperparameter vector to control this.

Since applying a multiparameter hierarchical clustering algorithm with different hyperparameter values may produce different partitions, we can view such algorithms as mapping a finite metric space $(X, d_X)$ to a function from the hyperparameter space to the space of partitions of $X$.  A \textbf{flattening} procedure then maps this function to a single partition. 

Many popular clustering algorithms, such as HDBSCAN \cite{mcinnes2017accelerated} and ToMATo \cite{chazal2013persistence}, include a flattening step that operates on the same intuition that drives the analysis of persistence diagrams in TDA. That is, the most important clusters (homological structures) of a dataset are those which exist at multiple scales (have large differences between their birth and death times). 

In this paper we will characterize and study clustering algorithms as functors, similarly to \cite{carlsson2008persistent,culbertson2016consistency,shiebler2020clustering}. We will particularly focus on multiparameter hierarchical clustering algorithms with partially ordered hyperparameter spaces. This perspective allows us to guarantee that the clustering algorithms we study preserve both non-expansive maps between metric spaces and the ordering of the hyperparameter space. Our contributions are:
\begin{itemize}
    \item We describe an algorithm for flattening multiparameter hierarchical clusterings, which we demonstrate is a functorial map from a category of multiparameter hierarchical partitions to a category of binary integer programs.
    \item We introduce a Bayesian update algorithm for learning a distribution over clustering hyperparameters from data.
    \item We prove that the composition of the Bayesian update algorithm and the flattening procedure is consistent.
\end{itemize}

\vspace{-5mm}
\section{Multiparameter Hierarchical Clustering}\label{2-hierarchical}
\vspace{-3mm}

In this work we will define flat clustering algorithms to map a finite metric space $(X, d_X)$ to a partition of $X$. We will primarily work with the following categories:
\begin{definition}\vspace{-2mm}
In the category $\met$ objects are finite metric spaces $(X, d_X)$ and morphisms are \textbf{non-expansive maps}, or functions $f: X \rightarrow Y$ such that $d_{Y}(f(x_1),f(x_2)) \leq d_X(x_1,x_2)$.
\end{definition}\vspace{-4mm}
\begin{definition}
In the category $\prt$ objects are tuples $(X, \px)$ where $\px$ is a partition of the set $X$. Morphisms in $\prt$ are functions $f: X \rightarrow Y$ such that if $S \in \px$ then $\exists S' \in Y, f(S) \subseteq S'$.
\end{definition}\vspace{-2mm}
We will also work in the subcategories $\metbij,\prtbij$ of $\met,\prt$ respectively in which morphisms are further restricted to be bijective.
\begin{definition}
Given a subcategory $\mathbf{D}$ of $\met$, a \textbf{flat clustering functor on $\mathbf{D}$} is a functor $F: \mathbf{D} \rightarrow \prt$ that is the identity on the underlying set $X$. In the case that $\mathbf{D}$ is unspecified we simply call $F$ a \textbf{flat clustering functor}. 
\end{definition}\vspace{-1mm}
Now recall that the set of connected components of the $\delta$-Vietoris-Rips complex of $(X, d_X)$ is the partioning of $X$ into subsets with maximum pairwise distance no greater than $\delta$. Given $a \in (0,1]$, an example of a flat clustering functor on $\met$ is the $a$-single linkage functor $\slink(a)$, which maps a metric space to the connected components of its $-log(a)$-Vietoris-Rips complex  \cite{shiebler2020clustering,culbertson2016consistency,carlsson2008persistent}. Given $a_1, a_2 \in (0,1]$, an example of a flat clustering functor on $\metbij$ is the \textbf{robust single linkage functor} $\slinkr(a_1, a_2)$ which maps a metric space $(X, d_X)$ to the connected components of the $-log(a_2)$-Vietoris-Rips complex of $(X, d^{a_1}_X)$ where:
\begin{gather*}
    %
    %
   d^{a_1}_X(x_1, x_2) = max(d_X(x_1, x_2), \mu_{X_{a_1}}(x_1), \mu_{X_{a_1}}(x_2))
\end{gather*}
and $\mu_{X_{a_1}}(x_1)$ is the distance from $x_1$ to its $\left \lfloor{a_1*|X|}\right \rfloor$th nearest neighbor \cite{NIPS2010_4068}. Intuitively, robust single linkage reduces the impact of dataset noise by increasing distances in sparse regions of the space. Note that robust single linkage is not a flat clustering functor on $\met$ because it includes a $k$-nearest neighbor computation that is sensitive to $|X|$.

Like single linkage and robust single linkage, many flat clustering algorithms are configured by a hyperparameter vector that governs their behavior. In the case that this hyperparameter vector is an element of a partial order $O$ we can represent the output of such an algorithm with a functor $O \rightarrow \prt$.

Recall that $\prt^{O}$ is the category of functors from $O$ to $\prt$ and natural transformations between them. We will write $\fprt{O}$ to represent the subcategory of such functors that commute with the forgetful functor $U: \prt \rightarrow \set$. Given $F: O \rightarrow \prt$ in $\fprt{O}$ we will call the image of $U \circ F$ the \textbf{underlying set} of $F$. Note that there also exists a forgetful functor $\fprt{O} \rightarrow \set$
that maps $F: O \rightarrow \prt$ to its underlying set and that any natural transformation in $\fprt{O}$ between the functors $F_X: O \rightarrow \prt$  and $F_Y: O \rightarrow \prt$ with underlying sets $X$ and $Y$ is fully specified by a function $f: X \rightarrow Y$.
\begin{definition}
Given a partial order $O$ and a subcategory $\mathbf{D}$ of $\met$, an \textbf{$O$-clustering functor on $\mathbf{D}$} is a functor $H: \mathbf{D} \rightarrow \fprt{O}$ that commutes with the forgetful functors from $\mathbf{D}$ and $\fprt{O}$ into $\set$. In the case that $\mathbf{D}$ is unspecified we simply call $H$ an \textbf{$O$-clustering functor}.
\end{definition}\vspace{-1mm}

For example, single linkage $\slink: \met \rightarrow \fprto$ is a $(0,1]^{op}$-clustering functor on $\met$ and
robust single linkage $\slinkr: \metbij \rightarrow \fprtbij{(0,1]^{op} \times (0,1]^{op}}$ is a $(0,1]^{op} \times (0,1]^{op}$-clustering functor on $\metbij$.

\begin{definition}
Given the functor $F_X \in \fprt{O}$ with underlying set $X$, its \textbf{partition collection} is the set $S_X$ of all subsets $S \subseteq X$ such that there exists some $a \in O$ where $S \in F_X(a)$.
\end{definition}\vspace{-1mm}
We will write the elements of $S_X$ (subsets of $X$) with the notation $S_X = \{s_{{1}_X}, s_{{2}_X}, \cdots, s_{{n}_X}\}$.

In practice it is often convenient to ``flatten'' a functor $F_X \in \fprt{O}$ to a single partition of $X$ by selecting a non-overlapping collection of sets from its partition collection $S_X$. Since we will express this selection problem as a binary integer program we will work in the following category:
\begin{definition}
The objects in $\bip$ are tuples $(n, m, c, A, B, u)$ where $n,m \in \mathbb{N}$, $c, u$ are $m$-element real-valued vectors, $A$ is an $n \times m$ real-valued matrix and $B$ is an $n \times m$ $\{0,1\}$-valued matrix. Intuitively, the tuple $(n, m, c, A, B, u)$ represents the following binary integer program: find an $m$-element $\{0,1\}$-valued vector $v$ that maximizes $c^{T}v$ subject to $Av + Bv \leq u$.

The morphisms between $(n, m, c, A, B, u)$ and $(n', m', c', A', B', u')$ are tuples $(P_c, P_u, P_A, P_{A^{*}}, P_B, P_{B^{*}})$ where
$P_u, P_A$ are $n' \times n$ real-valued matrices,
$P_{A^{*}}, P_c$ are $m \times m'$ real-valued matrices,
$P_B$ is an $n' \times n$ $\{0,1\}$-valued matrix
and
$P_{B^{*}}$ is an $m \times m'$ $\{0,1\}$-valued matrix such that:
\begin{gather*}
    P_c c' = c
    \qquad
    P_u u = u'
    \qquad
    P_A A P_{A^{*}} = A'
    \qquad
    P_B B P_{B^{*}} = B'
\end{gather*}
where the operation $P_B B P_{B^{*}}$ is performed with logical matrix multiplication. Intuitively, a morphism $(P_c, P_u, P_A, P_{A^{*}}, P_B, P_{B^{*}})$ maps the binary integer program ``find an $m$-element $\{0,1\}$-valued vector $v$ that maximizes $(P_c c')^{T}v$ subject to $Av + Bv \leq u$'' to the binary integer program ``find an $m'$-element $\{0,1\}$-valued vector $v$ that maximizes $c'^{T}v$ subject to $P_A A P_{A^{*}}v + P_B B P_{B^{*}}v \leq P_u u$''.
\end{definition}

When we construct a binary integer program from an object $F_X$ in $\fprt{O}$ with underlying set $X$, we weight the elements of its partition collection $S_X$ according to a model of the importance of different regions of $O$. In this work we will only consider $O$ that are Borel measurable, so we can represent this model with a probability measure $\mu$ over $O$. This probabilistic interpretation will be useful in Section \ref{learning} when we update this model with labeled data. We can then view the flattening algorithm as choosing a non-overlapping subset $\mathcal{P}_X \subseteq S_X$ 
(the elements of $\mathcal{P}_X$ are subsets of $X$ where no element of $X$ belongs to more than a single set in $\mathcal{P}_X$)
that maximizes the expectation of the function that maps $a$ to the number of $\sxi \in \mathcal{P}_X$ that are also in $F_X(a)$. Formally, the algorithm maximizes
$\int_{a \in O} 
\left|
\{\sxi \ |\ \sxi \in \mathcal{P}_X,
\sxi \in F_X(a)\}
\right|
d\mu$
. If $\mu$ is uniform this is similar to the Topologically Motivated HDBSCAN described in \cite{mcinnes2017accelerated}. 
\begin{proposition}\label{flatten}
Given a probability measure $\mu$ over $O$, there exists a functor  $\flatten: \fprtbij{O} \rightarrow \bip$ that maps $F_X: O \rightarrow \prtbij$ with partition collection $S_X$ to a tuple $(|S_X|,|S_X|,c,A,B,u)$ such that any solution to the problem ``find an $m$-element $\{0,1\}$-valued vector $v$ that maximizes $c^{T}v$ subject to $Av + Bv \leq u$'' specifies a non-overlapping subset $\mathcal{P}_X \subseteq S_X$ that maximizes
$\int_{a \in O} 
\left|
\{\sxi \ |\ \sxi \in \mathcal{P}_X,
\sxi \in F_X(a)\}
\right|
d\mu$.
\end{proposition}
\begin{proof}

Given a probability measure $\mu$ over $O$, $\flatten: \fprtbij{O} \rightarrow \bip$ maps the functor $F_X: O \rightarrow \prtbij$ with partition collection $S_X$ to the binary integer program $(|S_X|,|S_X|,c,A,B,u)$ where $c,u$ are $|S_X|$-element real-valued vectors and $A,B$ are respectively real-valued and $\{0,1\}$-valued $|S_X| \times |S_X|$ matrices where:
\begin{gather*}
    u_i = |S_X|
    \qquad
    c_i = \int_{\{a\ |\ \sxi \in F_X(a)\}}
    d\mu
    \\
    A_{i, j} = 
    \begin{cases}
        |S_X| - 1 & i=j
        \\
        0 & else
    \end{cases}
    \qquad
    B_{i, j} = 
    \begin{cases}
        1 & \sxi \cap \sxj \neq \varnothing
        \\
        0 & else
    \end{cases}
\end{gather*}
A natural transformation between $F_X, F_Y: O \rightarrow \prtbij$ with underlying sets $X, Y$ and partition collections $S_X,S_Y$ that is specified by the surjective function $f: X \rightarrow Y$ is sent to the tuple $(P_c, P_u, P_A, P^T_A, P_B, P^T_B)$ where
$P_c$ is an $|S_X| \times |S_Y|$ real-valued matrix, $P_A, P_u$ are $|S_Y| \times |S_X|$ real-valued matrices, and $P_B$ is an $|S_Y| \times |S_X|$ $\{0,1\}$-valued matrix such that:
\begin{gather*}
    P_{c_{i,j}}  = \begin{cases}
        \frac{
        \int_{\{a\ |\ \sxi \in F_X(a),\ \syj \in F_Y(a)\}} \ d\mu
        }{
        \int_{\{a\ |\ \syj \in F_Y(a)\}}
        \ d\mu  } & f(\sxi) \subseteq \syj 
        \\
        0 & else
    \end{cases}
    \qquad
    P_{u_{j, i}} = \begin{cases}
        \frac{|S_Y|}{|S_X|} & i = j \\
        0 & else
    \end{cases}
    \\
    P_{A_{j,i}} = \begin{cases}
        \sqrt{\frac{|S_Y| - 1}{|S_X| - 1}} & i = j \\
        0 & else
    \end{cases}
    \qquad
    P_{B_{j,i}}  = \begin{cases}
        1 & f(\sxi) \subseteq \syj
        \\
        0 & else
    \end{cases}
\end{gather*}
First we will show that any feasible solution to the integer program $\flatten F_X$ corresponds to a selection of elements from $S_X$ with no overlaps. If $\sxi \cap \sxj \neq \varnothing$, then the $i$th row of $A + B$ will have $|S_X|$ in position $i$ and $1$ in position $j$. This implies that if $v_i = 1$ then $(A+B)_{i}^{T} v \leq |S_X|$ if and only if $v_j = 0$. Note also that by the definition of binary integer programming this is the selection of elements that maximizes:
\begin{gather*}
    \sum_{\sxi \in \mathcal{P}_X}
    c_i
    =
    \sum_{\sxi \in \mathcal{P}_X}
    \int_{\{a\ |\ \sxi \in F_X(a)\}}
        d\mu
    =
    \int_{a \in O} 
    \left|
    \{\sxi \ |\ \sxi \in \mathcal{P}_X,
    \sxi \in F_X(a)\}
    \right|
    d\mu
\end{gather*}
Next, we will show that $\flatten$ is a functor. Consider $F_X, F_Y: O \rightarrow \prtbij$ with underlying sets $X, Y$ and partition collections $S_X,S_Y$ and suppose:
\begin{gather*}
    \flatten F_X = (|S_X|, |S_X|, c_X, A_X, B_X, u_X)
    \qquad
    \flatten F_Y = (|S_Y|, |S_Y|, c_Y, A_Y, B_Y, u_Y)
\end{gather*}
Consider also a natural transformation specified by the function $f: X \rightarrow Y$ and define the image of $\flatten$ on this natural transformation to be $(P_{c}, P_u, P_{A}, P^T_{A}, P_{B}, P^T_{B})$. We first need to show that:
\begin{gather*}
    P_{c} c_Y = c_X
    \qquad
    P_u u_X = u_Y
    \qquad
    P_{A} A_X P^{T}_{A} = A_Y
    \qquad
    P_{B} B_X P^{T}_{B} = B_Y
\end{gather*}
Where $P_{B} B_X P^{T}_{B} = B_Y$ is performed with logical matrix multiplication. In order to see that $P_{c} c_Y = c_X$, note that:
\begin{align*}
    (P_{c} c_Y)_{i} =
    %
    \sum_{\{j\ |\  f(\sxi) \subseteq \syj\}}
    \left(
    \int_{\{a\ |\ \syj \in F_Y(a)\}} d\mu\right)
    \left(\frac{
    \int_{\{a\ |\ \sxi \in F_X(a),\ \syj \in F_Y(a)\}} \ d\mu
    }{
    \int_{\{a\ |\ \syj \in F_Y(a)\}}
    \ d\mu  } \right)
    = \\
    %
    \sum_{\{j\ |\  f(\sxi) \subseteq \syj\}}
    \int_{\{a\ |\ \sxi \in F_X(a),\ \syj \in F_Y(a)\}} \ d\mu
    =
    \int_{\{a\ |\ \sxi \in F_X(a)\}}
    \ d\mu = 
    c_{X_i}
\end{align*}
Next, to see that $P_u u_X = u_Y$, note that
$
(P_u u_X)_{j}
=
\frac{|S_Y|}{|S_X|}|S_X|
=
|S_Y| 
= 
u_{Y_{j}}
$.
Next, to see that $P_{A} A_X P^{T}_{A} = A_Y$, recall that $P_{A}$ is an $|S_Y| \times |S_X|$ matrix and note that since $f$ is surjective it must be that $S_Y \leq S_X$. Therefore both the left $|S_Y| \times |S_Y|$ submatrix of $P_{A}$ and the top $|S_Y| \times |S_Y|$ submatrix of $P_{A}^T$ are diagonal, so the product $P_{A} A_X P^{T}_{A}$ is a diagonal $|S_Y| \times |S_Y|$ matrix with:
\begin{gather*}
    (P_{A} A_X P^{T}_{A})_{jj}
    =
    P_{A_{jj}} A_{X_{jj}} P^{T}_{A_{jj}}
    =
   \frac{|S_Y| - 1}{|S_X| - 1} (|S_X| - 1)
    =
    |S_Y| - 1
    = 
    A_{Y_{jj}}
\end{gather*}
Next, to see that $P_{B} B_X P^{T}_{B} = B_Y$, note first that $P_{B} B_X$ is a $\{0,1\}$-valued $|S_Y| \times |S_X|$ matrix where:
\begin{gather*}
    (P_{B} B_X)_{ji}
    =
    %
    \exists_{k=1...|S_X|}
    P_{B_{j, k}} \wedge B_{X_{k, i}}
    =
    %
    \exists\ \sxk \in S_X, f(\sxk) \subseteq \syj \wedge
    \sxk \cap \sxi \neq \varnothing
\end{gather*}
And therefore that:
\begin{align*}
    (P_{B} B_X P^{T}_{B})_{jj'}
    = \\
    %
    \exists\ \sxl \in S_X
    \left(
    \exists\ \sxk \in S_X, f(\sxk) \subseteq \syj
    \wedge
     \sxk \cap \sxl \neq \varnothing
    \right)
    \wedge
    \left(
    f(\sxl) \subseteq \syjp
    \right)
    = \\
    %
    \exists\ \sxl,\sxk \in S_X,
    f(\sxk) \subseteq \syj
    \wedge
    f(\sxl) \subseteq \syjp
    \wedge
    \sxk \cap \sxl \neq \varnothing
    = \\
    %
    %
    \syj \cap \syjp  \neq \varnothing
\end{align*}
Finally, note that $\flatten$ preserves the identity since $P_{B} = P_{A} = I$ when $S_X = S_Y$ and it preserves composition by the laws of matrix multiplication. 
\qed
\end{proof}

For example, if $\mu$ is uniform then the connected components of the Vietoris-Rips filtration of $(X, d_X)$ that have the largest differences between their birth and death times will be a solution to $(\flatten \circ \slink)(X, d_X)$. Note also that $\flatten$ is only functorial over $\fprtbij{O}$, and not all of $\fprt{O}$. Intuitively, this is because $\flatten$ maps natural transformations between elements of $\fprtbij{O}$ to linear maps that only exist when the functions underlying these 
natural transformations are bijective.

\vspace{-5mm}
\subsection{The Multiparameter Flattening Algorithm}

Given an $O$-clustering functor $H$ and a distribution $\mu$ over $O$ we can use Monte Carlo integration and $\flatten$ to implement the following algorithm:
\begin{algorithmic}[1]\label{flatteningalgorithm}
\Procedure{MultiparameterFlattening}{$H, \mu, (X, d_X), n$}
  \State Initialize an empty list $L$
  \State Repeat $n$ times:
  \State \indent Sample the hyperparameter vector $a$ according to $\mu$
  \State \indent Add each set in $H(X, d_X)(a)$ to $L$
  \State Define $S_X$ to be the list of unique elements of $L$
  \State Define $c$ such that $c_i$ is the number of times that $\sxi$ appears in $L$ 
  \State Set $A, B, u$ with $\flatten$
  \State Return the solution to the binary integer program $(|S_X|, |S_X|, c, A, B, u)$
\EndProcedure
\end{algorithmic}
We include an example of this procedure on Github
\footnote{https://github.com/dshieble/FunctorialHyperparameters}
that builds on McInnes et. al.'s \cite{mcinnes2017accelerated}'s HDBSCAN implementation. In Table \ref{performancetable} we demonstrate that applying this procedure and solving the resulting binary integer program can perform better than choosing an optimal parameter value.

\vspace{-5mm}
\begin{table}[]
    \centering
    \begin{tabular}{|c|c|c|}
        \hline
        Algorithm &
        \multicolumn{1}{|p{4.0cm}|}{\centering Adjusted Rand Score on \\ Fashion MNIST Dataset} &
        \multicolumn{1}{|p{4.0cm}|}{\centering Adjusted Rand Score on \\ 20 Newsgroups Dataset} \\
        \hline
        \hline
        \multicolumn{1}{|p{5.0cm}|}{\centering HDBSCAN With Optimal\\ Distance Scaling Parameter $\alpha$
        }
        & $0.217$ & $0.181$ \\
        \hline
        \multicolumn{1}{|p{5.0cm}|}{\centering HDBSCAN With Flattening Over\\ Distance Scaling Parameter $\alpha$
        }
        & $0.262$ & $0.231$ \\
        \hline
    \end{tabular}
    \caption{
    We compare the performance of applying $\flatten$ to HDBSCAN with simply running HDBSCAN with the value of the distance scaling parameter $\alpha$ that achieves the best Adjusted Rand Score.
    We evaluate over the Fashion MNIST \cite{xiao2017/online} and 20 Newsgroups \cite{lang1995newsweeder} datasets by using the scikit-learn implementation of the Adjusted Rand Score  \cite{scikit-learn}. We see that the $\flatten$ procedure performs consistently better, which suggests that it may be a good option for unsupervised learning applications such as data visualization or pre-processing. 
    }
\label{performancetable}
\end{table}
\vspace{-7mm}

\begin{figure}
\centering
\includegraphics[width=9cm,height=5cm]{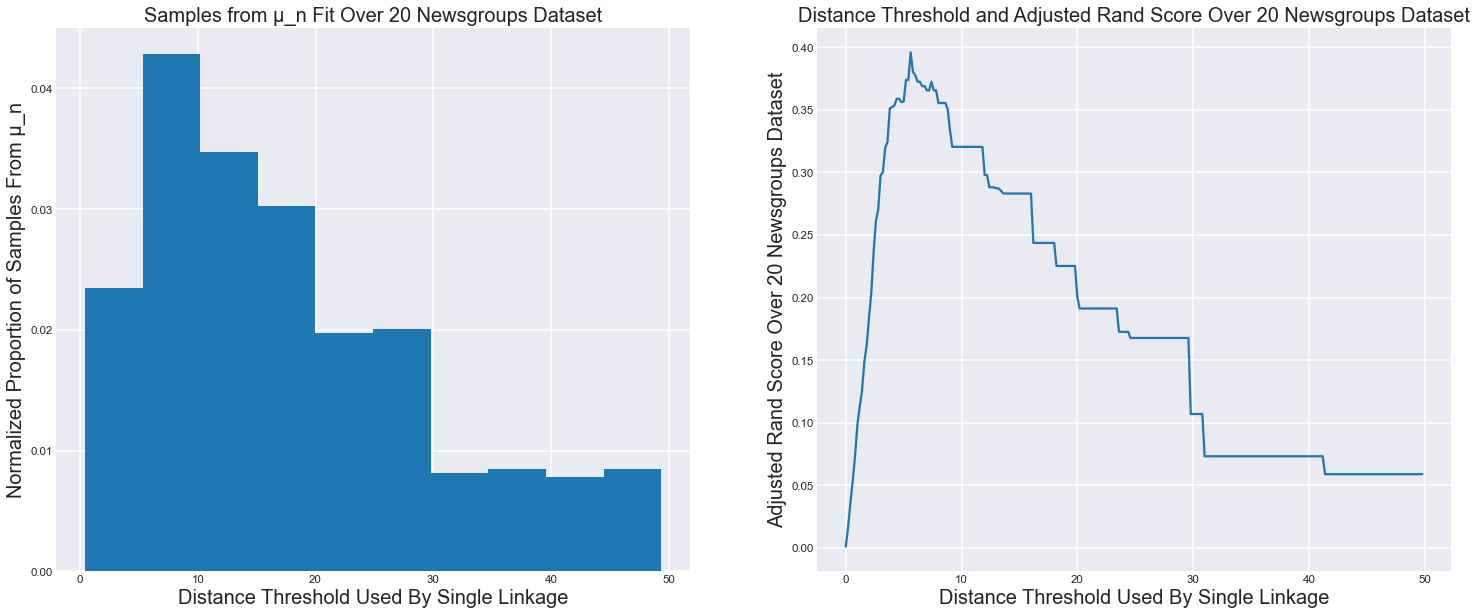}
\vspace{-2mm}
\caption{
If we define $H=\slink$ and apply Equation \ref{bayesianupdate} to learn $\mu_n$, then samples from $\mu_n$ are distance thresholds for single linkage. We see that the samples tend to be drawn from the region where the Adjusted Rand Score is highest (code on Github).}
\label{sampledist}
\end{figure}
\vspace{-2mm}

\vspace{-5mm}\subsection{Multiparameter Flattening with Supervision}\label{learning}
One of the most important components in the Multiparameter Flattening Algorithm is the choice of distribution $\mu$ over $O$. For example, if $O=(0,1]^{op}$ and $\mu$ is the Dirac distribution at $a$ then $\slink(X, d_X)(a)$ will be a solution to $(\flatten \circ \slink)(X, d_X)$.

We can leverage the importance of $\mu$ to enable our flattening procedure to learn from labeled data. Formally given an $O$-clustering functor $H$, a metric space $(X, d_{X})$, and an observed partition $\px$ of $X$ (the labels) we can use Bayes rule to define the probability measure $\mu_{\px}$ over $O$ where
$
%
\mu_{\px}(\sigma) =
    \frac{\int_{a \in \sigma}
    \gamma_{X}(\px \ |\ a)\  d\mu
    }{
    \int_{a \in O}
    \gamma_{X}(\px \ |\ a)\  d\mu
    }
$
if $\sigma$ is an element of the Borel algebra of $O$ and for each $a \in O$ the map $\gamma_{X}(\_ \ |\ a)$ is a probability mass function over the finite set of all partitions of $X$. There are several possible choices of $\gamma_{X}$. Intuitively, we want $\gamma_{X}(\px \ |\ a)$ to be large when $H(X, d_X)(a)$ and $\px$ are similar. The simplest choice would be
$
\gamma_{X}(\px \ |\ a) = \begin{cases}
    1 & H(X, d_X)(a) = \px
    \\
    0 & else
\end{cases}
$, but a more robust strategy would be to use the \textbf{Rand index}, which measures how well two partitions agree on each pair of distinct points \cite{rand1971objective}. That is:
\begin{gather}\label{randgamma}
%
\gamma_{X}(\px \ |\ a) =
\frac{|both(\px)| + |neither(\px)|}{
    \sum_{\px' \in \mathbf{P}_X}
    |both(\px')| + |neither(\px')|
}
\end{gather}
where $\mathbf{P}_X$ is the set of all partitions of $X$ and:
\begin{gather*}
    both(\px) = \{x_i, x_j \ |\ 
    \exists s_X \in \px, x_i, x_j \in s_X
    \wedge 
    \exists s'_X \in H(X, d_X)(a), x_i, x_j \in s'_X\}
    \\
    neither(\px) = \{x_i, x_j \ |\
    \not\exists s_X \in \px, x_i, x_j \in s_X
    \wedge 
    \not\exists s'_X \in H(X, d_X)(a), x_i, x_j \in s'_X\}
\end{gather*}
Note that by definition $\sum_{\px \in \mathbf{P}_X} \gamma_{X}(\px \ |\ a) = 1$.
Now suppose we have a set of tuples
$\{
(X_1, d_{X_1}, \mathcal{P}_{X_1}),
(X_2, d_{X_2}, \mathcal{P}_{X_2}),
\cdots,
(X_n, d_{X_n}, \mathcal{P}_{X_n})
\}$
where each $(X_i, d_{X_i})$ is a metric space and each $\pxi$ is a partition of $X_i$. 
Given an initial probability measure $\mu_{0}$ over $O$ we can use \textbf{Bayesian updating} to learn a posterior distribution over $O$ by iteratively setting:
\begin{gather}\label{bayesianupdate}
    \mu_{i}(\sigma) =
    \frac{\int_{a \in \sigma} \gamma_{X_i}(\pxi\ |\ a)
    \ d\mu_{i-1}
    }{
    \int_{a \in O} \gamma_{X_i}(\pxi\ |\ a)\  d\mu_{i-1}}
\end{gather}
In Figure \ref{sampledist} we show the results of this procedure. Under mild conditions the functor $Flatten_{\mu_{n}} \circ H$ maps  $(X, d_X)$ to an optimization problem with an optimal solution that is consistent with these $n$ observations. Formally:
\begin{proposition}\label{muflattenconvergence}
Given $d_1, d_2 \in \mathbb{N}$ suppose we have a compact region $R \subseteq \rl^{d_1}$ and an $O$-clustering functor $H$ where $O$ is a subset of $\rl^{d_2}$. Define $d_R$ to be Euclidean distance
and assume that:
\begin{itemize}
    \item $H$ is \textbf{$R$-identifiable}: for each $a \in O$ there exists some pair of $k$-element subsets $X_1, X_2 \subset R$ with $H(X_1, d_R)(a) \neq H(X_2, d_R)(a)$
    \item $H$ is \textbf{$R$-smooth}: for any $k$-element $X \subset R$ and $a \in O$ there exists a neighborhood $B_a$ of $a$ where for $a' \in B_a$ we have that $H(X, d_R)(a) = H(X, d_R)(a')$
\end{itemize}
Now suppose we sample $a_{*} \in O$ according to the uniform distribution $\mu_{0}$ over $O$ and for each $i>0$ uniformly select a $k$-element subset $X_i$ from $R$, set $(X_i, \mathcal{P}_{X_i}) = H(X_i, d_{R})(a_{*})$ and set $\mu_{i}$ as in Equation \ref{bayesianupdate}.

Then for any $k$-element $X \subset R$ there $\mu_0$-a.s. (almost surely) exists some $m$ such that for $n \geq m$, the partition $H(X,  d_{R})(a_{*})$ is an optimal solution to $(Flatten_{\mu_n} \circ H)(X,  d_{R})$.
\end{proposition}
\begin{proof}
We will use Doob's theorem \cite{doob49} to prove this.  Suppose $\mathcal{P}_{R}$ is the set of pairs $(X, \mathcal{P}_X)$ of finite $k$-element subsets $X \in R$ and partitions $\mathcal{P}_X$ of $X$. Define $\lambda_R$ to be the uniform distribution on the set of finite $k$-element subsets of $R$, and for each $a \in O$ define
$
\eta_a(\sigma) = 
\int_{(X, \mathcal{P}_X) \in \sigma}
\gamma_X(\mathcal{P}_X \ |\ a)\ d\lambda_R
$
where $\sigma$ is an element of the Borel algebra of $\mathcal{P}_{R}$. Now since $H$ is $R$-identifiable, $a \neq a' \implies \eta_a \neq \eta_{a'}$ and Theorem 2.4 in \cite{miller2018detailed} holds. Therefore for any $k$-element $X \subset R$, ball $B_{a_*}$ centered at $a_*$, and $\epsilon > 0$, there $\mu_0$-a.s. (almost surely) exists some $m$ such that for $n \geq m$ we have that $\mu_{n}(B_{a_*}) \geq 1 - \epsilon$. Therefore, no solution to $(Flatten_{\mu_n} \circ H)(X,  d_{R})$ can be improved by including sets that only exist in partitions of $X$ formed from $H(X,  d_{R})(a')$ where $a' \not\in B_{a_*}$. Since $H$ is $R$-smooth, this implies that the partition $H(X,  d_{R})(a_{*})$ is an optimal solution to $(Flatten_{\mu_n} \circ H)(X,  d_{R})$
\qed\end{proof}

\vspace{-5mm}
\section{Discussion and Next Steps}\label{5-discussion}\vspace{-3mm}

One issue with the $\flatten$ procedure is that it reduces to a binary integer program, which can be very slow to solve. In the case that the hyperparameters of a multiparameter hierarchical clustering algorithm form a total order, we can organize its outputs in a merge tree or dendrogram, and then solve the optimization problem with a tree traversal \cite{chazal2013persistence,mcinnes2017accelerated}.

However, it is not always possible to use this strategy on hierarchical clustering algorithms that accept multiple real-valued hyperparameters (e.g. the hyperparameter space is a general subset of $\rl^d$ where $d > 1$). In this case it is possible that there exist clusterings $c \in H(a)$ and $c' \in H(a')$ that are neither disjoint nor in a containment relationship.
There are some ways to get around this limitation, however.  For example, Rolle and Scoccola \cite{rolle2020stable} index hyperparameters by a curve $\gamma$ in $\rl^d$, rather than all of $\rl^d$. In this way the hyperparameter space remains a total order. In the future we plan to explore other strategies to solve the flattening optimization problem efficiently.

\bibliographystyle{splncs04}
\bibliography{main}






\end{document}